\documentclass[letterpaper]{article} 
\usepackage{aaai2026}  
\usepackage{times}  
\usepackage{helvet}  
\usepackage{courier}  
\usepackage[hyphens]{url}  
\usepackage{graphicx} 
\usepackage{natbib}  
\usepackage{caption} 
\usepackage{float} 
\usepackage{microtype}
\usepackage{amsmath} 
\usepackage{booktabs} 
\usepackage{xcolor} 
\usepackage{amsthm} 
\usepackage[linesnumbered,ruled,vlined]{algorithm2e}
\usepackage{subfig} 
\usepackage{newfloat}
\usepackage{listings}
\DeclareCaptionStyle{ruled}{labelfont=normalfont,labelsep=colon,strut=off} 
\floatstyle{ruled}
\newfloat{listing}{tb}{lst}{}
\floatname{listing}{Listing}
\newtheorem{theorem}{Theorem}
\newtheorem{corollary}[theorem]{Corollary}%

\newtheorem{thm}{\protect\theoremname}
\newtheorem{defn}[thm]{\protect\definitionname}
\newtheorem{prop}[thm]{\protect\propositionname}
\newtheorem{lem}[thm]{\protect\lemmaname}

\providecommand{\claimname}{Claim}
\providecommand{\definitionname}{Definition}
\providecommand{\lemmaname}{Lemma}
\providecommand{\propositionname}{Proposition}
\providecommand{\theoremname}{Theorem}
\newcommand{\descr}[1]{\noindent \textbf{#1}}

\title{Approximation Algorithm for Constrained $k$-Center Clustering:\\ A Local Search Approach}
\author{
    Chaoqi Jia\textsuperscript{\rm 1},
    Longkun Guo\textsuperscript{\rm 2,3}\thanks{Corresponding Author},
    Kewen Liao\textsuperscript{\rm 4},
    Zhigang Lu\textsuperscript{\rm 5},
    Chao Chen\textsuperscript{\rm 1},
    Jason Xue\textsuperscript{\rm 6}
}
\affiliations{
    \textsuperscript{\rm 1} School of Accounting, Information Systems and Supply Chain, RMIT University, Melbourne, Australia\\
    \textsuperscript{\rm 2} School of Mathematics and Statistics, Fuzhou University, China\\
    \textsuperscript{\rm 3} School of Integrated Circuits, Guangdong University of Technology, China\\
    \textsuperscript{\rm 4} School of Information Technology, Deakin University, Australia\\
    \textsuperscript{\rm 5} Western Sydney University, Australia\\
    \textsuperscript{\rm 6} CSIRO's Data61 and Responsible AI Research (RAIR) Centre, Adelaide University, Australia\\
    
    chaoqi.jia@student.rmit.edu.au, longkun.guo@gmail.com, kewen.liao@deakin.edu.au,
    z.lu@westernsydney.edu.au, chao.chen@rmit.edu.au, jason.xue@data61.csiro.au
}

\usepackage{bibentry}

\begin{document}
\maketitle
\begin{abstract}
Clustering is a long-standing research problem and a fundamental tool in AI and data analysis. The traditional $k$-center problem, known as a fundamental theoretical challenge in clustering, has a best possible approximation ratio of $2$, and any improvement to a ratio of $2 - \epsilon$ would imply $\mathcal{P} = \mathcal{NP}$.
In this work, we study the constrained $k$-center clustering problem, where instance-level cannot-link (CL) and must-link (ML) constraints are incorporated as background knowledge. Although general CL constraints significantly increase the hardness of approximation, previous work has shown that disjoint CL sets permit constant-factor approximations. However, whether local search can achieve such a guarantee in this setting remains an open question. To this end, we propose a novel local search framework based on a transformation to a \textit{dominating matching set} problem, achieving the best possible approximation ratio of $2$. The experimental results on both real-world and synthetic datasets demonstrate that our algorithm outperforms baselines in solution quality. 
\end{abstract}

\begin{links}
    \link{Code}{https://github.com/ChaoqiJia/LSCKC}
\end{links}

\section{Introduction}
Clustering is a fundamental task in unsupervised machine learning, and it is often employed in semi-supervised learning applications through background knowledge~\cite{van2020survey}. It is widely used in applications such as online customer segmentation~\cite{10chen2012data} and recommendation systems~\cite{27sarwar2002recommender}. Typical clustering problems include $k$-means~\cite{bidaurrazaga2021k}, $k$-center~\cite{alipour2020approximation} and facility location~\cite{garimella2015scalable}, which are commonly $\mathcal{NP}$-hard but with broad applications. In particular, we consider the general metric $k$-center problem where the best possible approximation ratio is well known to be 2 by~\citet{hsu1979easy}. The problem aims to minimize the maximum distance between any data point and its closest center. This objective can be alternatively considered as minimizing the (maximum) covering radius of the clusters, with application scenarios like minimizing the furthest traveling distance in trip planning or the largest dissimilarity in document summarization. Several decades ago, two independent works~\cite{gonzalez1985clustering,hochbaum1985best} both achieved a $2$-approximation on $k$-center clustering using different approaches. 

Furthermore, in reality, relationships existing among entities can be partially derived from some accessible background knowledge in a semi-supervised setting, such as similar entities, disjoint entities, or more concretely, the available entity class labels. The relationships are generally represented as instance-level cannot-link (CL) and must-link (ML) constraints, which can be leveraged as a form of regularization to obtain better clusters. Clustering with instance-level constraints was first introduced by~\citet{wagstaff2000clustering}, where both CL and ML constraints between each pair of data nodes were considered, i.e., nodes with an ML constraint must belong to the same cluster, whereas nodes with a CL constraint must be placed into different clusters. When using clustering to directly infer the classification of data~\cite{basu2008constrained}, both CL and ML constraints that express side information about the underlying class structure can greatly boost the classification performance. In semi-supervised learning, labeling unlabeled data via clustering with CL and ML constraints can help labels propagate more effectively~\cite{li2009constrained}. This paper focuses on the problem of $k$-center clustering with CL and ML constraints, for which we design a local search algorithm that achieves the best possible $2$-approximation.

\paragraph{Related Work}
Instance-level constraints such as CL and ML have been widely adopted in different clustering problems such as $k$-means clustering~\cite{wagstaff2001constrained}, spectral clustering~\cite{coleman2008spectral}, and hierarchical clustering~\cite{davidson2009using}. In addition, \citet{basu2008constrained} had collated an extensive list of constrained clustering problems. CL constraints were known to be problematic to clustering as their inclusion could lead to a computationally intractable feasibility problem~\cite{davidson2007complexity}, i.e., no efficient approximation algorithm can exist for finding any clustering result satisfying all CL constraints unless $\mathcal{P} = \mathcal{NP}$. This inapproximability result is obtained via a reduction from the $k$-coloring problem and has hindered the development of approximation algorithms for constrained clustering. Despite the negative theoretical result, CL constraints are widely useful for modeling or resolving conflicting relationships.
For instance, \citet{davidson2010sat} were among the first to study the  $k$-center clustering problem with instance-level constraints, focusing on the case $k = 2$. They incorporated an SAT-based framework to obtain an approximation scheme with $(1 + \epsilon)$-approximation for this extreme case. After that, \citet{brubach2021fairness} studied $k$-center only with ML constraints and achieved a $2$-approximation. Recently, \citet{guo2024efficient} leveraged the structure of disjoint CL sets and then the construction of a reverse dominating set (RDS) to result in a constant factor approximation ratio of $2$ that is the best possible.

In contrast with their works, we propose a different technique, called \textit{dominating matching set} (DMS), based on local search to solve this problem. In particular, local search is an effective approach for solving many optimization problems \cite{arya2001local,kanungo2002local,cohen2022improved,feng2019improved}, though it is often non-trivial to provide approximation guarantees.  By the local search algorithm for the CL $k$-center clustering, our method achieves the best $2$-approximation ratio unless $\mathcal{P} = \mathcal{NP}$ and provides better cost for this problem in the experiments.

\paragraph{Our Results and Contribution}
In this paper, we mainly deal with the troublesome CL constraints for the general value of $k$, i.e., the CL $k$-center problem. Following the constraint types by~\citet{guo2025near}, we consider approximating $k$-center with both disjoint CL and ML constraints. We develop a unique local search framework by first assuming that the radius of an optimal solution to the constrained (CL/ML) $k$-center is known. With the optimal radius, the CL set constraints can be transformed into an auxiliary $l$-partite graph, and the $k$-centers are to be locally searched through a structure called \textit{single-swap free dominating matching set} (SF-DMS). Our local search algorithm first selects a feasible solution to CL $k$-center, and then repeatedly finds a single swap therein to improve the solution until no such single swap exists. In the approximation ratio analysis, we bound the number of remaining centers by $k$ when the local search terminates. The key idea is to show that the centers actually compose an SF-DMS, for which the radius is at most twice the optimal. We also observe that it is $\mathcal{NP}$-hard to compute a \textit{minimum} DMS, whereas an SF-DMS can be computed in polynomial time and leads to a desirable solution to the problem with a provable performance guarantee. Our main contribution can be summarized as follows: 
\begin{itemize}
\item Propose a threshold-based local search algorithm for $k$-center with CL/ML constraints by novelly transforming a CL $k$-center problem into a \textit{dominating matching set} (DMS) in an auxiliary $l$-partite graph.

\item Show that the best possible 2-approximation can be achieved for the constrained $k$-center problem with disjoint CL constraints by introducing a structure called the single-swap free DMS with local search, which enables the computation of a desirable solution in polynomial time. 

\item Conduct extensive experiments on both real-world and synthetic datasets to demonstrate the algorithm's superior performance in practice in terms of the clustering cost and practical approximation ratio. The results show that our local-search-based approximation algorithm consistently outperforms the baselines under both disjoint and intersected CL constraints. 
\end{itemize}

\section{Preliminary}
\label{Sec:Pre}

Given an input dataset $P =\{ p_{1},\dots,p_{n}\}$ with a collection of $n$ data points, we first denote the distance function between any two data points $p_i$ and $p_j$ as $d(p_i, p_j)$. Then the distance between a point $p \in P$ and the center set $C$ can be defined as $d(p,\,C)=\min_{c\in C}d(p, c)$, where the distance is $ d(p, \emptyset) = +\infty$ while the set $C$ is empty. The classical (discrete) $k$-center problem aims to select a center set $C \subseteq P$ such that the maximum metric distance between $P$ and $C$ is minimized. Formally, the objective of the traditional $k$-center clustering is to seek an optimal center set $C^* = \arg\min_{C \subseteq P}\max_{p\in P}d(p,\,C)$ with a minimum radius. For brevity, we denote the optimal radius by $opt$. 

\paragraph{Constrained (CL/ML) $k$-Center}
With the same optimization objective as the traditional $k$-center, the constrained $k$-center problem differs in that its input data points are subject to predefined CL and ML constraints. We introduce necessary definitions and notations for formulating the constrained (CL/ML) $k$-center following \citet{guo2024efficient}. Formally, we set cannot‐link (CL) constraints $\mathcal{Y} = \{Y_1,\dots,Y_l\}$, where each $Y_i \subseteq P$ and $|Y_i|\le k$ is a set of points that cannot be assigned to the same cluster. Similarly, must‐link (ML) constraints are given as a collection of sets $\mathcal{X} = \{X_1,\dots,X_h\}$, where each $X_i\subseteq P$ is a clique of points that must all lie in the same cluster. Let $C(p)$ denote the cluster (center) a point $p$ is assigned to; then the CL and ML constraints can be respectively translated into the conditions that all $Y_{i} \in \mathcal{Y}$ satisfy $\forall (p,q) \in Y_{i}: C(p)\ne C(q)$ and all $X_{i} \in \mathcal{X}$ satisfy $\forall (p,q) \in X_{i}: C(p)=C(q)$. In addition, by construction, CL sets can be either mutually disjoint or not, implying dramatically different approximation analysis. The disjoint CL setting can be represented as $\forall Y_{i},\,Y_{j}\in\mathcal{Y}:Y_{i}\cap Y_{j}=\emptyset$. Note that even some intersected CLs can be reduced. For instance, if there are CLs $\{u,v\}$ and $\{v,w\}$ and an ML $\{u,w\}$, it can be reduced into a single CL $\{v,z\}$ with $z$ being a merged node of $u$ and $w$. On the other hand, ML sets are naturally mutually disjoint, as otherwise any two intersected ML sets can be merged into a single ML set due to transitivity. In this paper, we use CL/ML $k$-center to denote the $k$-center problem with CL/ML constraints.

\begin{lem}\label{lem:ML}\cite{guo2025near}
ML $k$-center has a 2-approximation ratio.
\end{lem}
  \begin{algorithm}[t]
    \small
     \SetAlgorithmName{Algorithm}{Algorithm}{List of Algorithms}
        \SetAlCapFnt{\normalfont}
     \caption{{Approximation algorithm of CL/ML $k$-center. (\citeauthor{guo2025near})}}
     \label{alg:ML}
      \KwIn{Database $P$ of size $n$ with ML sets ${\mathcal{X}}$ and CL sets ${\mathcal{Y}}$ a given threshold $\eta = 2opt$, and a positive integer $k$.}
      \KwOut{Center set $C$ with radius bounded by $\eta$.}
         Set $C\leftarrow Y$, which is the CL set of maximum size in $\mathcal{Y}$\;
         \For{each point $p\in P$}{ 
         Add $p$ to the center set $C$ if it satisfies one of the following two cases:\\
        \begin{itemize}
        \item [ a)] $p\in P\setminus \bigcup_{X\in \mathcal{X}} X$ and $d(p,C)>\eta$\;   
        \item [ b)] There exists $\max_{x\in X}d(x,C) > \eta$,\\
        where $X$ is an ML set that contains $p$;
        \end{itemize}
        }
    Return $C$.
\end{algorithm}

\section{Approximations for CL $k$-Center}
\label{sec:appro_cl_kcenter}
In this section, we first deduce the properties of the threshold-based algorithm for $k$-center with ML constraints~\cite{guo2025near} to serve our analysis. Then, regarding the same threshold, we introduce an auxiliary graph called an $l$-partite graph to represent the relationship between the CL sets and the current centers, and then define a structure called \textit{dominating matching set} (DMS) therein. Based on this structure, we further aim to find a solution with an approximation ratio of $2$ by incorporating local search techniques into the threshold-based algorithmic framework.

For simplicity of the algorithmic description, we first assume that the objective value of an optimal solution is known as $opt$, and use $2opt$ as the threshold. Then, we show that the ratio $2$ guarantee still holds even without knowing $opt$.

\begin{lem}\label{lem:ML-num}
Let $\mathcal{U}=\left\{U_{1},\dots, U_{k}\right\}$ be the set of clusters of an optimal solution wrt $opt$. The centers of  $C$ produced by Alg.~\ref{alg:ML} must belong to different clusters of $\,\mathcal{U}$, when $\eta = 2opt$.
\end{lem}
\begin{proof}
For the first iteration of Alg.~\ref{alg:ML} (in Step 1), the lemma is obviously true since the center set $C$ is a CL set, which limits the points to belong to different clusters. Suppose the lemma holds for the center set $C^i$ of the first $i$ iterations; we only need to demonstrate the lemma for the point $p$ that is added $C^{i+1}$ in the $(i+1)$th iteration by analyzing the following two cases:
\begin{enumerate}
\item[(a)]  If $p\notin X$, for  $\forall X\in \mathcal{X}$:

According to the algorithm,  the distance between $p$ and $C^i$, $d(p,C^i)$, is greater than $\eta = 2opt$. So $p$ must be in an optimal cluster uncovered by the previous centers in~$C^i$. 
\item[(b)]  Otherwise (i.e., there exists $X\in \mathcal{X}$ with $p\in X$):
    
According to the algorithm, suppose $X$ (and hence $p\in X$) is in an optimal cluster covered by one of the centers in $C^i$, but $\max_{x\in X}d(x,C^i) > 2opt$ holds. This contradicts the fact that each diameter of optimal cluster $U_i$ is less than or equal to $2opt$.
\end{enumerate}
Thus, $p$ is in a different optimal cluster from the covered clusters by the center set $C^i$. Additionally, the points in the center set $C^{i+1}$ also belong to a different optimal cluster. When $i = k-1$, this completes the proof. 
\end{proof}

\subsection{Local Search  for CL $k$-Center}
\label{subsec:LS_cl_kcenter}
We give a local search method for individually processing the disjoint CL sets by using an auxiliary graph to represent the relationship between the CL sets and the set of centers, which essentially transforms the problem into finding a special structure therein. Again, we assume that $opt$ is known and use $\eta=2opt$ as the threshold.

\subsubsection{Dominating Matching Set}
First, we define a structure  to capture the relationship between the current center set $C$ and a CL set $Y$:
\begin{defn}
\label{defn:minmax-matching}
(Threshold-based maximum-cardinality matching) Given a center set $C$ and a CL set $Y$, for a threshold $\eta>0$, we define an auxiliary bipartite graph $G=(C,Y;E)$ in which, for $c\in C$ and $y\in Y$,  an edge $e(c,y)$ exists if and only if it is within a distance $d(c,y) \leq \eta$. Then $M(G,\eta)\subseteq E$ is a threshold-based maximum matching if and only if: (1) $M(G,\eta)$ is a matching; (2) $M(G,\eta)$ is with the largest cardinality, i.e. $|M(G,\eta)|$ is maximized.
\end{defn}

Then, for the relationship of the points constrained by CL sets, we use a more sophisticated auxiliary graph as follows:

\begin{defn}\label{def:kpartitiongraph-2}(The auxiliary $l$-partite graph) 
Let ${\mathcal{Y}}=\{Y_{1},\dots,Y_{l}\}$ be the collection of disjoint CL sets and $\eta$ be a given distance bound. 
We say $G_{{\mathcal{Y}}}$ is the auxiliary $l$-partite graph wrt $\eta$ if and only if:
\begin{itemize}
 \item  The vertex set  of  $G_{{\mathcal{Y}}}$ is $V(G_{{\mathcal{Y}}})=V_{{\mathcal{Y}}}=\bigcup_{Y\in{\mathcal{Y}}}Y$; 
 \item The edge set of  $G_{{\mathcal{Y}}}$ is  $E(G_{{\mathcal{Y}}})=E_{{\mathcal{Y}}}$. For a pair of points  $y\in Y_{i}$ and $y'\in Y_{j}$ for $i\neq j$,  $E_{{\mathcal{Y}}}$ contains edge $e\left(y,\,y'\right)$  if and only if $d\left(y,\,y'\right)\leq \eta$  holds. 
\end{itemize}
\end{defn}

Essentially, our algorithm aims to identify a set of points $\Gamma\subseteq V_{{\mathcal{Y}}}$,  such that all points within $V_{{\mathcal{Y}}}$ can be clustered regarding the CL constraints. 
Formally, such  $\Gamma$ can be  defined as below:
\begin{defn}
\label{LMI}(Dominating matching set, DMS) For an auxiliary $l$-partite graph $G_{{\mathcal{Y}}}$ regarding $\eta$,  the current center set $\Gamma\subseteq V_{{\mathcal{Y}}}$ is a \emph{dominating matching set}, if and only if for each $Y\in{\mathcal{Y}}$, there exists a threshold-based maximum matching $M$ between $\Gamma\setminus Y$ and $Y\setminus\Gamma$ with its cardinality being exactly $\vert M\vert =\vert Y\setminus \Gamma\vert$.
\end{defn}

We say a DMS $\Gamma$ is \textit{infeasible} wrt $\eta$, if and only if under the threshold $\eta$ there exists $Y$ that the threshold-based maximum-cardinality matching between $ Y\setminus \Gamma$ and  $\Gamma\setminus Y$ has size strictly smaller than $\vert Y\setminus \Gamma\vert$.
In particular, the \emph{minimum} DMS problem aims to find a \emph{dominating matching set} with the smallest cardinality. A natural idea of the algorithm is to find a \textit{minimum} DMS wrt $G_{{\mathcal{Y}}}$ to serve the points of $V_{\mathcal{Y}}$. However, a \textit{minimum} DMS is hard to compute, due to its hardness as stated below:

\begin{prop}\label{prop:ptas}
Even when $\vert Y_{i}\vert=1, \forall i$,  and the points are distributed on a $2$-dimensional plane, a \textit{minimum} DMS is $\mathcal{NP}$-hard to compute. Moreover, it admits no EPTAS (i.e., Efficient PTAS with runtime $f(\frac{1}{\epsilon})poly(n)$) under the exponential time hypothesis.
\end{prop}

Clearly, when $\vert Y_{i}\vert=1, \forall i$, and the points are distributed on a plane, the \emph{minimum} DMS problem still includes the problem of finding a dominating set in a given disk graph, which is to find a minimum number of disks to dominate a set of given points distributed on the plane. In light of the negative result by \citet{marx2007optimality}, the above-mentioned problem, and the \textit{minimum} DMS problem, consequently is $\mathcal{NP}$-hard and admits no EPTAS under the exponential time hypothesis. 

Differently, we say a {DMS} $\Gamma$ is \emph{minimal} if the removal of any point therein makes $\Gamma$ no longer a {DMS}. Although it can be calculated in polynomial time, a \emph{minimal} DMS is yet not sufficient to approximate the problem.  So we further find a special \emph{minimal} DMS, called single-swap free {DMS} obtained by employing the local search technique, to accomplish the approximation task.

\begin{defn}
\label{SFDMS}(Single-swap free dominating matching set, SF-DMS) A DMS $\Gamma\subseteq V_{{\mathcal{Y}}}$ is single-swap free, if and only if for each $p\in V_{{\mathcal{Y}}}$ and $\{u,v\}\subseteq \Gamma$, the set $\Gamma\cup\{p\} \setminus {\{u,\,v\}}$ is not a DMS.
\end{defn}

In the above definition, we call the operation of adding a point $p$  while removing $u$ and $v$ an enhanced single swap, which is slightly different from the commonly used definition in the local search approach. Moreover, an {SF-DMS} is inherited by a \textit{minimal} DMS, since single-swap free also means no single point could be removed from $\Gamma$ while keeping $\Gamma$ a feasible cluster (e.g., when $p \in {\{u,\,v\}}$). In contrast, with the hardness of the minimum version, it is polynomial-time solvable to find an SF-DMS in a given auxiliary $l$-partite graph $G_{{\mathcal{Y}}}$. Moreover, we have the following lemma, which states that an {SF-DMS} can be employed to approximate our problem.

\begin{lem}\label{lem:DMSforoptimal} 
Let ${\mathcal{U}}$ be the clusters of an optimal solution wrt $opt$. Assume that  ${\mathcal{Y}}$ is a family of CL constraint sets and $\Gamma$ is an SF-DMS of the auxiliary $l$-partite graph $G_{{\mathcal{Y}}}$ wrt distance $2opt$. Then, we have $\vert\Gamma\vert\leq k_{\mathcal{Y}}$, where $k_{\mathcal{Y}}$ is the number of clusters of  ${\mathcal{U}}$  containing points of $V_{{\mathcal{Y}}}=\bigcup_{Y\in{\mathcal{Y}}}Y$.
\end{lem}

\begin{proof}
We show the lemma by induction. When $\vert{\mathcal{Y}}\vert=1$, the theorem is apparently true because on the one hand, the only CL set $Y\in{\mathcal{Y}}$ has its points appearing in exactly $\vert Y\vert$ clusters of ${\mathcal{U}}$ and hence $k_{\mathcal{Y}}=|Y|$; while on the other hand, $\Gamma$ exactly contains all points of $Y$ and hence $\vert\Gamma\vert=\vert Y\vert=k_{\mathcal{Y}}$.

Next, we need to show the case for $\vert{\mathcal{Y}}\vert=l$. For the task, we consider $Y$ to reduce the size of ${\mathcal{Y}}$ for induction. Assume ${\mathcal{U}}_{Y} \subseteq {\mathcal{U}}$ is the set of optimal clusters, where each cluster $U\in {\mathcal{U}}_{Y}$ contains a distinct element of $Y$. So $|{\mathcal{U}}_{Y}|=|Y|$ holds. Suppose $|\Gamma|\geq k_\mathcal{Y}+1$, then we will show a contradiction for the following two cases, considering whether $\Gamma$ contains at least one distinct element for each  $U\in \mathcal{U}$ or not:

\begin{enumerate}
\item[(a)] $\Gamma\cap U\neq\emptyset$, $\forall  U\in{\mathcal{U}}_{Y}$: 

We will show that $\Gamma$ contains at least one center that can be removed. 
Denote by $\Gamma_{Y}\subseteq\Gamma$  the set of centers in which exactly one center appears in each $U\in {\mathcal{U}}_{Y}$  and $\Lambda=\bigcup_{U\in {\mathcal{U}}_{Y}}U$. Clearly, $|\Gamma_Y|=\vert{\mathcal{U}}_{Y}\vert=|Y|$. Let ${\mathcal{Y}_{res}}=\{Y'\setminus\Lambda\mid Y'\in{\mathcal{Y}\setminus}\{Y\}\}$. Then, similar to the notation of $V_{\mathcal{Y}}$, we define $V_{\mathcal{Y}_{res}}=\bigcup_{Y\in\mathcal{Y}_{res}}Y$, where the points of $V_{{\mathcal{Y}}}\setminus\Lambda$ are constrained by ${\mathcal{Y}_{res}}$.   That is, we actually divide the points of $V_{\mathcal{Y}}$ into two parts $\Lambda$ and $V_{\mathcal{Y}_{res}}$, for which the two DMS are respectively $\Gamma_{Y}$ and $\Gamma\setminus\Gamma_{Y}$. By definition of  $\Gamma_{Y}$,   $\Gamma\setminus\Gamma_{Y}$ is with a size $|\Gamma|-|\Gamma_Y|=|\Gamma|-|Y|$.
By assumption,   $|\Gamma|\geq k_{\mathcal{Y}}+1$ holds, so we have 

\begin{equation}
|\Gamma\setminus\Gamma_{Y}|\geq k_{\mathcal{Y}}+1-|Y|.\label{eq:GammaY}
\end{equation} 

On the other hand, $\Gamma\setminus\Gamma_{Y}$ is an SF-DMS, since any enhanced single swap thereon is also an SF-DMS for $\Gamma$. By the \emph{induction hypothesis},

\begin{equation}
\vert\Gamma\setminus\Gamma_{Y}\vert\leq k_{\mathcal{Y}_{res}}, \label{eq:induction}
\end{equation}

where $k_{\mathcal{Y}_{res}}$ is the number of clusters of ${\mathcal{U}}$  containing points of $V_{{\mathcal{Y}_{res}}}$.
Because the centers of $\Gamma_Y$ completely cover the subset of clusters of $V_{\mathcal{Y}}$ where they appear,  which are exactly $|\Gamma_Y|$ clusters. So, there are at most $k_{\mathcal{Y}}-|\Gamma_Y|$ clusters of $V_{\mathcal{Y}}$ that contain points of $V_{\mathcal{Y}_{res}}$. So we get $k_{\mathcal{Y}_{res}}\leq k_{\mathcal{Y}}-|\Gamma_Y|$  by the definition of ${\mathcal{Y}_{res}}$. 
Combining  with the Inequality (\ref{eq:induction}), we immediately get: 

\begin{equation}
\vert\Gamma\setminus\Gamma_{Y}\vert\leq k_{\mathcal{Y}}-|\Gamma_Y|= k_{\mathcal{Y}}-|Y|,\label{eq:boundy'}
\end{equation}

which contradicts Inequality (\ref{eq:GammaY}).\\

\item[(b)] Otherwise: 

Assume that $\Gamma'$ is a subset of $\Gamma$ that has exactly one center appearing in each $U\in \mathcal{U}$.
Let $\mathcal{U}'$ be the set of optimal clusters, each of which respectively contains a distinct element of $\Gamma'$, i.e., $|\Gamma'| = |U'|$. Similarly, we define $\Lambda=\bigcup_{U\in \mathcal{U}'} U$. 
Different from Case (a),  by assumption, there must exist at least one set, say $U_e\in \mathcal{U}\setminus\mathcal{U}'$ with $\Gamma\cap U_e=\emptyset$. We will show that for any point $p\in U_e$, there exists $u,v\in \Gamma\setminus \Gamma'$, such that $\Gamma\cup\{p\}\setminus \{u,v\}$ remains a feasible center set regarding the distance bound $2opt$ and the CL constraints. This is an enhanced single swap for $\Gamma$ by definition and contradicts the assumption that $\Gamma$ is an SF-DMS.

It remains to show the existence of such $\{u,v\}$ in $\Gamma\setminus \Gamma'$. We define $\mathcal{Y}_{res}=\{Y'\setminus(\Lambda\cup U_e)\mid Y'\in \mathcal{Y}\}$. Clearly, $\Gamma\setminus\Gamma'$ is an SF-DMS of $\mathcal{Y}_{res}$, since any enhanced single swap that is feasible for $\mathcal{Y}_{res}$, is also feasible for $\Lambda$. Then, following the same line with the proof of Case (a)  and observing that the points of $\mathcal{Y}_{res}$ appear only in clusters of $\mathcal{U}\setminus\mathcal{U}'-U_e$ up to $k_{\mathcal{Y}}-|\Gamma'|-1$.
Therefore, adding up the center $p$ for $U_e$, we get that $\Gamma\setminus\Gamma'$ have at most $k_{\mathcal{Y}}-|\Gamma'|$ centers:

\begin{equation}|\Gamma\setminus \Gamma' | \leq  k_{\mathcal{Y}} -\vert\Gamma'\vert.\label{eq:induction-2}
\end{equation}

On the other hand, by assumption, we have

\[|\Gamma|\geq k_{\mathcal{Y}}+1.\] 

Then, by the definition of $\Gamma',$ we have

\[|\Gamma\setminus \Gamma' |=|\Gamma|-\vert\Gamma'\vert\geq  k_{\mathcal{Y}}+1 -\vert\Gamma'\vert.\]

So, combining Inequality (\ref{eq:induction-2}), we immediately obtain a contradiction and consequently complete the proof.
\end{enumerate}
\end{proof}

\begin{algorithm}[t]
        \SetAlgorithmName{Algorithm}{Algorithm}{List of Algorithms}
        \SetAlCapFnt{\normalfont}
         \small
        \caption{{Local search for CL $k$-center.}}
        \label{alg:localsearch}
        \KwIn{A family of CL constraint sets ${\mathcal{Y}}$, a given distance bound $\eta$, and a set of center candidates  $\Gamma$ that already covers $V_{\mathcal{Y}}$ with radius bounded by $\eta$.}
        \KwOut{An SF-DMS $\Gamma$ wrt $\eta$.}
      
       \While{true}{
            Find  $p \in V_{\mathcal{Y}}$ and  a pair of points  $u,\,v\in \Gamma$, such that $\Gamma\cup \{p\}\setminus \{u,v\}$ remains a center set  that  covers  $V_{\mathcal{Y}}$ wrt  $\eta$ and the CL constraints\;
            \uIf{there exist pair such $\{p\}$ and $\{u,v\}$ as above}{
                Set $\Gamma\leftarrow \Gamma\cup\{p\}\setminus\{u,v\}$\;
            }\Else{Return $\Gamma$.}
        }
\end{algorithm}

\begin{algorithm}[t]
     \SetAlgorithmName{Algorithm}{Algorithm}{List of Algorithms}
        \SetAlCapFnt{\normalfont}
    \small
     \caption{{Grand algorithm for constrained $k$-center.}}
     \label{alg:k-center}
      \KwIn{Database $P$ of size $n$ with ML sets ${\mathcal{X}}$ and CL sets ${\mathcal{Y}}$, a positive integer $k$, and a distance bound $\eta = 2opt$.}
      \KwOut{A center set $C$ with radius bounded by $2opt$.}
      Set $C_{1}\leftarrow \emptyset$ and $C_{2}\leftarrow \emptyset$\;
      \tcp{\small STAGE 1(a): Centers for ML}
      Find a set of points $C_{1}$ as centers by Alg.~\ref{alg:ML} respecting $P$,  ${\mathcal{X}}$ and $2opt$ but ignoring ${\mathcal{Y}}$\;
      \tcp{\small STAGE 1(b): Find a set of center candidates for CL sets {such that $ C_{1}\cup C_{2}$ covers all points of $\mathcal{Y}$}}
      Set $C \gets C_1$\;
      \For{each bipartite graph $G(Y_i,C; E)$ regarding $Y_i\in \mathcal{Y}$ and  $\eta=2opt$}{
      \For{each $p_{i}\in{Y_i} \cap X_j, X_j\in \mathcal{X}$}{
            Set the weight of $d(p_i,C)$  in $G(Y_i,C; E)$ using the value of $\min_{c\in C}\max_{x\in X_j}d(x,c)$\;
        }
         Find a threshold-based maximum-cardinality matching $M\in G(Y_i,C; E)$ respecting $\eta=2opt$ as defined as in Definition~\ref{defn:minmax-matching}\; 
         \If{$|M\cap Y_i|<|Y_i|$}{
           Set $C\leftarrow  (Y_i\setminus {M\cap Y_i}) \cup C$\;
        }
      } 
          
       \tcp{\small \textbf{Stage 2:} Centers for CL sets by shrinking $C_{2}$ regarding $C_{1}$ and $C_{2}$ through tuning Alg.~\ref{alg:localsearch} of finding an SF-DMS as above.}
        Set $C_2 \gets C\setminus C_1$\;
       \While{true}{
        Find  $p \in V_{\mathcal{Y}}$ and a pair of points  $\{u,\,v\}\in C_2$, such that $ C_1\cup C_2\cup \{p\}\setminus \{u,v\}$ remains a set of centers that feasibly cover  $V_{\mathcal{Y}}$ wrt  $2opt$ and satisfy the CL constraints\;
        
       \eIf{there exist no such $\{p\}$ and $\{u,v\}$ as above}{ 
        Return $C_{1}\cup C_{2}$ and terminate.
       }{
       Set $C_2\leftarrow C_2\cup\{p\}\setminus\{u,v\}$\;
       }
       }
\end{algorithm}

\subsubsection{The Local Search Algorithm}
According to Lem.~\ref{lem:DMSforoptimal}, we can calculate an SF-DMS with a size bounded by $k$ to cluster the points of $V_{\mathcal{Y}}$. Note that we actually have $P=V_{\mathcal{Y}} $ if considering each unconstrained point of $P$ as a new single CL set $Y$ with $|Y|=1$. Hence, we consider only the points constrained by CL sets through the following process:
(1) Initially, find a feasible solution $\Gamma$ to CL $k$-center with a size possibly larger than $k$; (2) Execute the enhanced single swap as follows if there exists any: find a point $p\in V_{\mathcal{Y}}\setminus \Gamma $ and a pair of points $u,v\in \Gamma$, such that $\Gamma\cup{\{p\}}\setminus \{u,\,v\}$ remains a valid cluster respecting the radius $2opt$.    
Formally, the layout of our algorithm is as in Alg.~\ref{alg:localsearch}. From Lem.~\ref{lem:DMSforoptimal}, we can immediately obtain the approximation ratio of Alg.~\ref{alg:localsearch} as in the following corollary.

\begin{corollary}
    Alg.~\ref{alg:localsearch} produces $\Gamma$ that is a set of centers with size $|\Gamma|\leq k$ and can cluster the points of $P$ within radius $2opt$ and with all the CL constraints satisfied.
\end{corollary}

It is worth noting that Lem.~\ref{lem:DMSforoptimal} holds for SF-DMS, but is not true for $minimal$ DMS as demonstrated in the counter-example depicted in the full version. 

\subsection{Approximating  CL/ML k-Center}

Now that we have introduced all necessary components, we present our algorithm for the $k$-center problem with both CL and ML constraints. The key idea is to treat each ML-constrained point set as a ``big'' point, redefining its distance computation, and then combine the two subroutines (Algs.~\ref{alg:ML} and~\ref{alg:localsearch}) so that the total number of clusters remains bounded by $k$. Formally, the complete procedure appears in Alg.~\ref{alg:k-center}.

For the quality of solution produced by Alg.~\ref{alg:k-center}, we have:
\begin{thm}
\label{thm:correctness}Let $opt$ be the radius of an optimal solution. If $\eta \geq 2opt$, then Alg.~\ref{alg:k-center} produces $C$ with $|C|\leq k$. 
\end{thm}

\begin{proof}
Let $ \mathcal{U}=\left\{ U_{1},\dots,U_{k}\right\} $ be the clusters of an optimal solution wrt $opt$.  We will divide $C$ into two parts according to $\mathcal{U}$, say $C_1$ and $C_2$, and then bound their size individually.  Assume that we have added $k_{u}$ centers in  Stage 1(a) as $C_1$. Following Lem. \ref{lem:ML-num}, clearly $k_{u}\leq k$ holds.

It remains only to bound the size of  $C_2=C\setminus C_1$ with $k-k_u$. Because all the points outside the CL sets can be clustered to the $k_u$ centers of $C_1$, $C_2$ needs only to serve a subset of points in the CL sets. So we will identify $C_2$ as an SF-DMS regarding a subset of CL-constrained points that are not captured by the first $k_u$ centers.  
  
Since each center  of  $C_1$ is in a different cluster of  $\mathcal{U}$ as in Lem.~\ref{lem:ML-num}, we can w.l.o.g. assume that  the  $k_{u}$ centers of $C_1$ appear in  $\{U_{1},\dots,U_{k_u}\}$, respectively. Then, each set in $\{U_{k_u+1},\dots,U_{k}\}$ do not contain any center  of $C_1$. Let $V_{res}= \bigcup_{i=k_u+1}^{k}U_i$ and  $\mathcal{Y}'=\{Y\cap V_{res} \mid Y\in \mathcal{Y} \}$. 
 Clearly, after the local search phase for the CL center candidates in $C_2$ as in Stage $2$, the points of $V_{\mathcal{Y}'}$ remained in $C_2$   clearly compose an  SF-DMS for $V_{\mathcal{Y}'}$ (or for a subset of $V_{\mathcal{Y}'}$). The reason is that a feasible enhanced single swap exists for $C_2$ and $V_{\mathcal{Y}'}$ is also a feasible enhanced single swap for   $C_2$ and   $V_{\mathcal{Y}}$.  So when  Alg.~\ref{alg:k-center} terminates when there exists no  feasible enhanced single swap for   $C_2$ and   $V_{\mathcal{Y}}$ (and therefore neither for $C_2$ and   $V_{\mathcal{Y}'}$).  
 Moreover, because  $V_{\mathcal{Y}'}$ is the set of CL-constrained points appearing only in the sets of $\{U_{k_u+1},\dots,U_{k}\}$, points of $V_{\mathcal{Y}'}$ appear in at most $k - k_u$ clusters of the optimal  ${\mathcal{U}}$. 
 Following Lem.~\ref{lem:DMSforoptimal}, the SF-DMS $C_2$ contains at most $k - k_u$ points (centers). Therefore, the total number of centers remaining in $C$ is at most $k$.
\end{proof}

Following Thm.~\ref{thm:correctness}, we immediately have the approximation ratio 2  of Alg.~\ref{alg:k-center}. Next, we analyze the runtime of Alg.~\ref{alg:k-center}.
\begin{lem}
\label{lem:alg5runtime}
Alg.~\ref{alg:k-center} runs  in time  $O(n^2k^{4.5})$.
\end{lem}
\begin{proof}
 Firstly, Stage 1(a) of Alg.~\ref{alg:k-center} takes $O(kn)$ to find the centers without considering cannot links regarding $2opt$; secondly, Stage 1(b) takes $O((|Y|+|C|)^{1/2}\cdot(|Y|\cdot|C|))=O(|Y|\cdot k^{1.5})$ time for computing the maximum matching for each $Y$ using~\citeauthor{hopcroft1973n}. Since there are $\sum_{ Y\in \mathcal{Y}}|Y|=n$ points in total, this stage consumes $O(nk^{{3/2}})$ time overall. For Stage 2, the while-loop repeats for $O(k)$ iterations, and each iteration computes $O(nk^2)$ swap pairs and maximum matching for checking the feasibility of the updated set $C_2\cup\{p\}\setminus\{u,v\}$. Each such feasibility check takes  $O(nk^{{1.5}})$ time across all CL sets. So the total runtime is $O(n^2k^{{4.5}})$. 
\end{proof}
\subsection{Without Knowing $opt$}
By integrating Alg.~\ref{alg:k-center} with the algorithm and analysis of Sec.~IV.B in \citet{guo2025near}, we immediately recover the same factor 2 guarantee without knowing~$opt$. For clarity, we restate this method.
Although we do not have the exact value of $opt$ and the exact value of $opt$  might require exponential time to compute, we can achieve the same ratio based on the basic observation as in the following lemma.
\begin{lem}\label{lem:rlays}
Let $\Psi=\{d_{ij}\vert p_{i},p_{j}\in P\}$, where $d_{ij} = d(p_i,p_j)$. Then we have $opt\in \Psi$. In other words, the optimal radius $opt$  must be a distance between two points of~$P$. 
\end{lem}
\begin{proof}
Suppose the lemma is not true. Then there must exist at least a distance in  $\Psi$  not larger than $opt$, since otherwise, any pair of points cannot be in the same cluster. Let $r=\max\{d_{ij}\mid d_{ij}\in\Psi,\,d_{ij}<opt\}$. Then under distances $r$ and $opt$, every center can cover an identical set of points in the same manner,  because every point covered by center $i$ under distance  $opt$ can also be covered by  $i$ under distance $r$ (as $d_{ij}\leq r$ if $d_{ij}<opt$), and vice versa. So $r<opt$ can be used as a smaller radius of the clustering, a contradiction.
\end{proof}
That is, we need only to find the smallest $r\in \Psi$, such that under distance bound $2r$, our algorithm can successfully return $C$ with $\vert C\vert\leq k$. Clearly, such $r$ is bounded by $opt$ because of Thm.~\ref{thm:correctness}.
\section{Experimental Evaluation}
\label{sec:exp}
In our experiments, we evaluate clustering performance on $3$ real-world datasets and a synthetic dataset. For the synthetic dataset, we report the empirical approximation ratio against baselines. Next, we examine how performance varies as we adjust the constraint ratio and the repetition ratio of intersected CL sets. The details on clustering runtime and experiments with varying $k$ will be provided in the full version.

\subsection{Configurations}
\label{sec:exp_config}
\descr{Datasets.} We use three UCI datasets (Cnae-9~\cite{cnae-9_233}, Skin~\cite{skin_segmentation_229} and Covertype~\cite{covertype_31}) to evaluate the clustering quality following the previous works. Due to the $\mathcal{NP}$-hardness of constrained $k$-center, we use the simulated datasets ($k = 50$) provided by~\citet{guo2024efficient} to evaluate the empirical approximation ratio. The details of the datasets can be found in the full version. 

\descr{Constraint Construction.} 
We construct the constraints following the methodology in Sec. V.A of~\citet{guo2025near}. Specifically, we use both intersected and disjoint CL/ML constraints on real-world datasets to evaluate clustering performance, and disjoint CL/ML constraints on synthetic datasets to assess the empirical approximation ratio of our algorithm relative to baseline methods.
\begin{figure}[t]
\centering
\includegraphics[width=\linewidth]{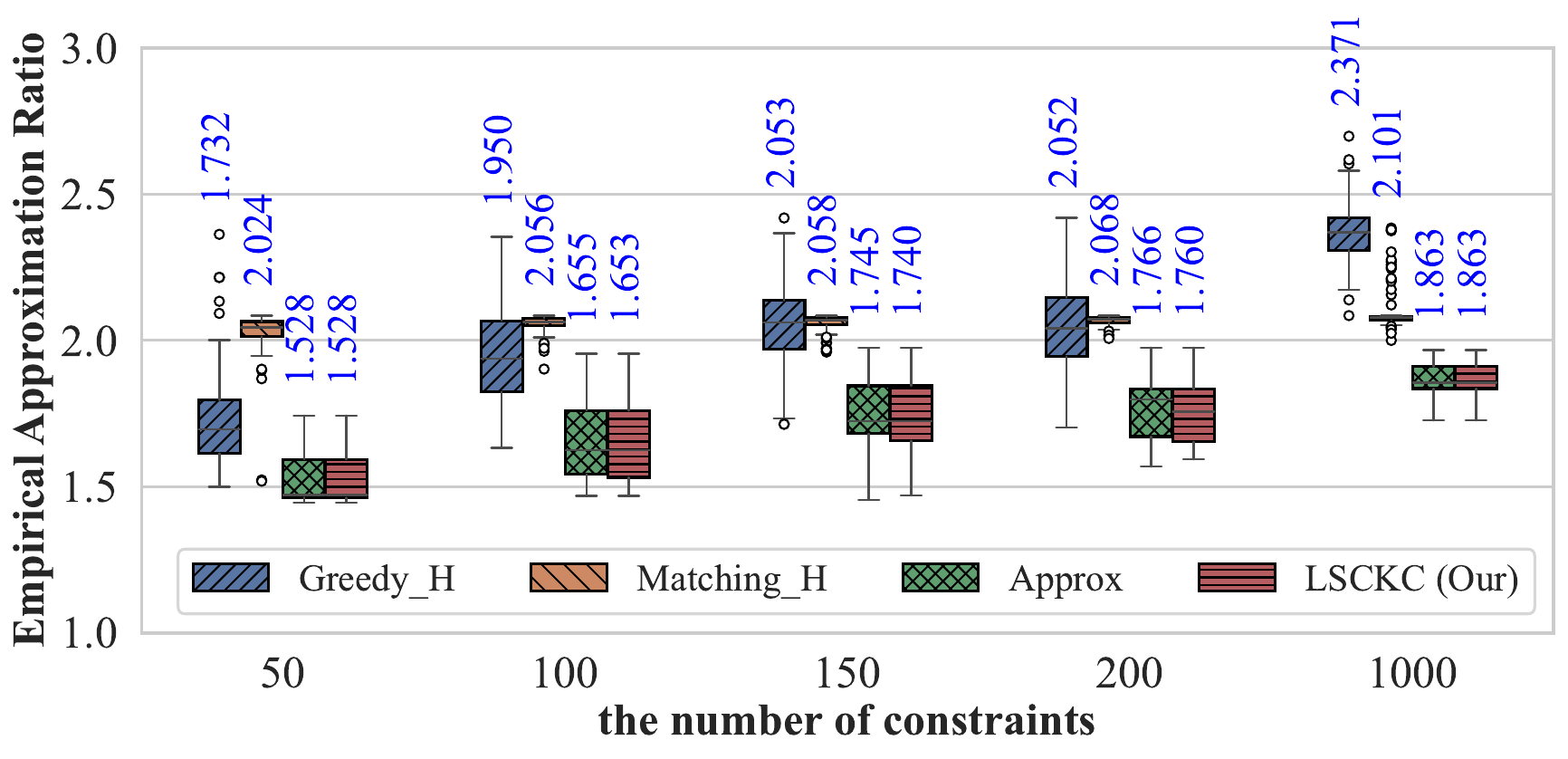}
\caption{Empirical approximation ratio on the synthetic dataset with $k = 50$. (The blue values indicate the average clustering performance for each algorithm with varying the number of constraints.)}
\label{fig:simulated-data-cost}
\end{figure}

\begin{figure*}[!t]
\centering
\subfloat[Disjoint CL/ML $k$-center clustering with varying constraint ratios.]{\includegraphics[width = 0.97\textwidth]{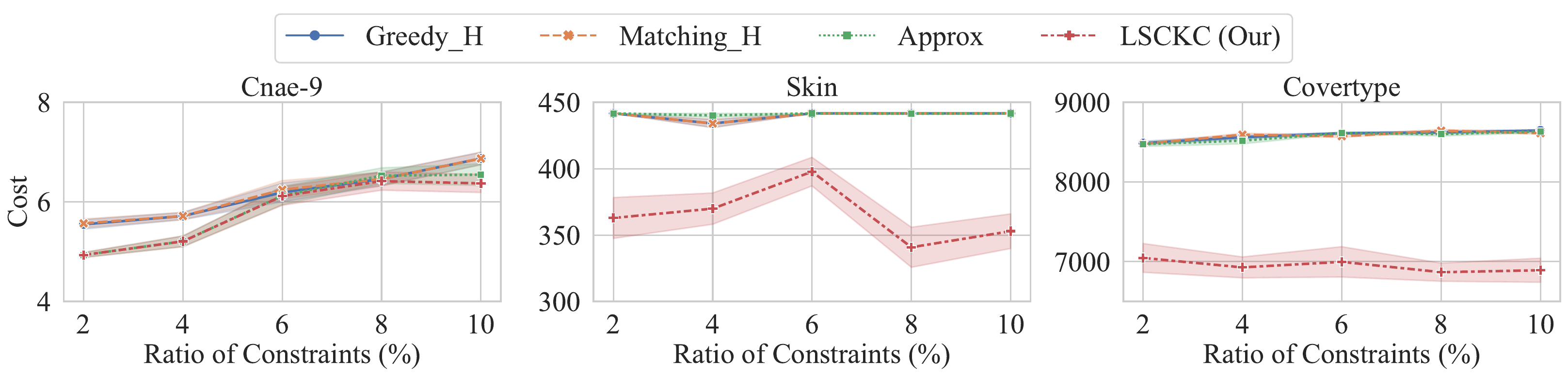}
\label{subfig:varying_constr}
}
\hfill
\subfloat[Intersected CL/ML $k$-center clustering with varying repetition ratios at a 10\% constraint level.]{\includegraphics[width = 0.97\textwidth]{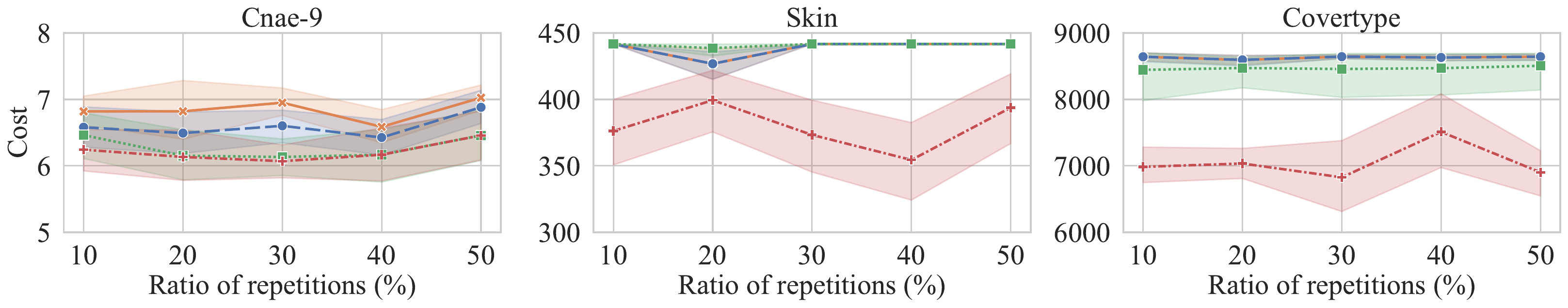}
\label{subfig:varying_repeti}
}
\caption{Cost of the constrained $k$-center clustering.}
\label{fig:cost_real}
\end{figure*}

\descr{Algorithms.}
We compare our method with two heuristic algorithms (referred to as \textit{Greedy\_H} and \textit{Matching\_H}) and an approximation algorithm (called \textit{Approx}) introduced by~\citet{guo2024efficient} as baselines. Moreover, we evaluate the performance of our local search constrained $k$-center algorithm (denoted as \textit{LSCKC}), as detailed in Alg.~\ref{alg:k-center}.

\descr{Evaluation Metrics.} In line with previous studies, we employ the clustering quality metrics for the experiments, which include \textit{Cost} and running time.

\descr{Experimental Settings.} 
With disjoint CL/ML constraints, we varied the proportion of constrained points to evaluate the clustering cost. In the intersected CL/ML constraints, we initially selected a fraction (10\%) of points from the dataset, then sampled $x\%$ of those points, where the value $x\%$ ranges from 0\% to 50\%  (plotted on the x‑axis). Each configuration was repeated at least 20 times, and we reported the average results. All algorithms were implemented in Java 1.8.0 and executed on an Apple M1 Max CPU with 32 GB of RAM.

\subsection{Experimental Results and Analysis}
\label{subsec:exp_ra}

\subsubsection{Empirical Approximation Ratio\\}
Fig.~\ref{fig:simulated-data-cost} illustrates the empirical worst-case approximation ratios calculated from the experiments of Approx, Matching\_H, and Greedy\_H~\cite{guo2025near} and our algorithm (LSCKC) on the simulated dataset. Since the simulated dataset was constructed with a given optimal radius $opt$, we report the maximum empirical approximation ratios as $r_{\text{max}}/opt$, where $r_{\text{max}}$ is the maximum cluster radius (cost function of the constrained $k$-center problem) in the clustering result of each algorithm. Based on the results under the different ratios of CL/ML constraint sets, we confirm that both Approx and LSCKC (Alg.~\ref{alg:k-center}) consistently attain the maximum empirical approximation ratios under 2, \textit{which matches our theoretical result proved in Thm.~\ref{thm:correctness}.}

In addition, when comparing the average empirical approximation ratios with the baselines, which are indicated by the blue values in Fig.~\ref{fig:simulated-data-cost}, we observe that our local search algorithm (LSCKC) consistently achieves the minimum empirical approximation ratio. This improvement can be attributed to the swap steps in the local search process, which further refine the set of centers and effectively reduce the cost in the constrained $k$-center problem. Moreover, as the number of constraints increases, the average empirical approximation ratios also tend to grow. We attribute this phenomenon to the incremental addition of constraints, which reduces the feasible region and renders the optimization problem increasingly restrictive.

\paragraph{Clustering quality for the disjoint/intersected constraints} 
In Fig.~\ref{fig:cost_real}, we compare the clustering costs of LSCKC, Approx, Matching\_H, and Greedy\_H on three real-world datasets, considering both disjoint and intersected CL/ML constraints. In Subfig.~\ref{subfig:varying_constr}, we present clustering costs at varying disjoint constraint ratios from 2\% to 10\%, with a fixed number of clusters $k=30$. In Subfig.~\ref{subfig:varying_repeti}, we show clustering costs at varying the ratio of repetitions from 10\% to 50\% for the intersected constraints, with a fixed constraint ratio of 10\%. From these results, we conclude that \textit{LSCKC (Alg.~\ref{alg:k-center}) consistently achieves effective solutions for the constrained $k$-center problem on real-world datasets}. We attribute this improvement to the swap steps employed in our local search, which further optimize the clustering cost under constrained conditions. Notably, the baseline (Approx) mainly selects center sets based on CL constraints, while our local search steps enable the algorithm to identify better centers from points that are not restricted by the existing constraints. This flexibility allows our method to outperform the baselines. For example, in the Cnae-9 dataset, the advantage of our algorithm becomes more apparent with more constraints, since additional constraints further limit the solution space for the Approx baseline, but our local search can explore and select the centers from the unconstrained points.

According to Fig.~\ref{fig:cost_real}, we observe that LSCKC achieves better performance on the Skin and Covertype datasets. In our experimental setting, we adopt constraint generation methods following previous work, which results in constraint sets of relatively small size. We believe that the advantages of our algorithm are particularly evident on datasets with shorter constraint lengths, as handling constraints is more difficult in such cases, while our algorithm is well-suited for managing these challenges.

\section{Conclusion}
In this paper, we explored approximation algorithms for the constrained $k$-center problem incorporating instance-level cannot-link (CL) and must-link (ML) constraints. Although general CL constraints introduce inherent inapproximability challenges, we focused specifically on cases with disjoint CL constraints, following previous work, and developed a novel constant-factor approximation algorithm. Our proposed algorithm achieves an optimal approximation ratio of 2 by integrating a threshold-based strategy within a local search framework. Experimental evaluations on both real-world (including scenarios with disjoint/intersected constraints) and synthetic datasets demonstrate the superior performance of our local search approximation algorithm.

\section*{Acknowledgments}

This work is supported by National Natural Science Foundation of China (No. 12271098) and Key Project of the Natural Science Foundation of Fujian Province  (No. 2025J02011).

\bibliography{aaai2026}
\end{document}